\documentclass[letterpaper]{article}
\usepackage[utf8]{inputenc}
\usepackage{amsmath}
\usepackage{amsthm}
\usepackage{amssymb}
\usepackage{graphicx}
\usepackage[margin=1in,letterpaper]{geometry}
\usepackage{enumitem}
\usepackage{xcolor}
\usepackage[colorlinks=true,allcolors=blue]{hyperref}
\usepackage{subcaption}
\usepackage{algorithm}
\usepackage[authoryear]{natbib}
\usepackage{comment}

\newtheorem{theorem}{Theorem}
\newtheorem{corollary}[theorem]{Corollary}

\DeclareMathOperator*{\argmin}{arg\,min}

\newcommand{\cI}{\mathcal{I}}
\newcommand{\cD}{\mathcal{D}}
\newcommand{\cC}{\mathcal{C}}
\newcommand{\cS}{\mathcal{S}}
\newcommand{\Ex}{\mathbf{E}}
\newcommand{\Px}{\mathbf{P}}

\begin{document}

\title{Beam Search for Feature Selection}
\author{Nicolas Fraiman \\  \href{mailto:fraiman@email.unc.edu}{fraiman@email.unc.edu} 
\and Zichao Li \\ 
\href{mailto:lizichao@live.unc.edu}{lizichao@live.unc.edu}}
\date{}
\maketitle

\begin{abstract}
In this paper, we present and prove some consistency results about the performance of classification models using a subset of features. In addition, we propose to use beam search to perform feature selection, which can be viewed as a generalization of forward selection. We apply beam search to both simulated and real-world data, by evaluating and comparing the performance of different classification models using different sets of features. The results demonstrate that beam search could outperform forward selection, especially when the features are correlated so that they have more discriminative power when considered jointly than individually. Moreover, in some cases classification models could obtain comparable performance using only ten features selected by beam search instead of hundreds of original features.
\end{abstract}

\section{Introduction}\label{sec:intro_beam}
In many machine learning problems, the data can be represented by a matrix, in which the rows represent observations and the columns represent features. High-dimensional data, where the number of features is very large, is challenging for machine learning models in several aspects:
1) Some machine learning models might not work in its original form and need to be modified accordingly, such as by adding regularization terms.
2) Even if the machine learning model works for high-dimensional data, it might induce huge computational cost in terms of runtime and memory.
3) Many features might be noisy or redundant, therefore overfitting might become a serious issue.
4) The model might be hard to interpret because of the large number of features.

One common way to deal with high-dimensional data is to perform dimensionality reduction, which can be further divided into \emph{feature extraction} and \emph{feature selection}. Feature extraction constructs new features by transforming the original features, and includes methods like principal component analysis (PCA) \citep{wold1987principal}, random projection \citep{bingham2001random}, and autoencoder \citep{kramer1991nonlinear}. Feature selection selects a subset of the original features, and includes different types of methods discussed in the next paragraph. Since feature selection does not transform the original features, the result of feature selection is more interpretable than the result of feature extraction, and the interpretability provided by feature selection methods is desirable and useful in many applications. For example, in gene expression analysis, the data might contain hundreds or less samples and thousands or more genes, and the task is to predict whether a sample is a cancer tissue or not using the expression levels of all the genes. In this case, feature selection methods can be used to identify a small subset of genes that are considered as relevant or important for the prediction task, which in turn might provide useful biological insights into the relationship between cancer and the selected genes.

According to \citet{chandrashekar2014survey}, feature selection methods can be divided into three categories:
\begin{itemize}
   \item Wrapper methods use predictive performance to evaluate and select feature subsets. Since exhaustive search of all feature subsets is generally infeasible, usually some heuristic search algorithm is used to construct the feature subsets. Examples of search algorithms include forward selection, backward elimination, branch and bound, simulated annealing, and genetic algorithms \citep{guyon2003introduction}. Wrapper methods are usually very computationally intensive, because a new model is trained for each feature subset. However, they also provide the best predictive performance in most cases, because the predictive performance of the feature subset is directly being optimized.
   \item Filter methods use some relevance criteria to rank the features and select the highly ranked features. Examples of relevance criteria include Fisher score, Relief, and mutual information based criteria \citep{tang2014feature}. Compared to wrapper methods, filter methods are less computationally intensive because the relevance criteria are usually easier to compute. However, they also typically have worse predictive performance, because the relevance criteria is independent of the predictive model, so there is no guarantee on the performance of the predictive model using the features selected based on the relevance criteria.
   \item Embedded methods perform feature selection as part of the model construction process. One typical example is adding a $L_1$ regularization term $\lambda \sum_{i=1}^{p} |w_i|$ to the loss function of a linear model with coefficients $w_i$ \citep{tibshirani1996regression}. As the regularization parameter $\lambda$ increases, many coefficients would shrink to exactly zero, and the features with nonzero coefficients can be considered as being selected by the model. Decision trees like CART also have a built-in mechanism to perform feature selection \citep{breiman1984classification}.
\end{itemize}

Among different wrapper methods, forward selection is perhaps the simplest approach: adding features one by one, and at each step selecting the feature that leads to the best predictive performance when combined with the previously selected features. However, forward selection has one obvious drawback: only keeping the best subset of features at each step during the process of searching might be too greedy. For example, forward selection is not able to select the best discriminating pair of features if the pair does not include the best discriminating feature. A more general approach is beam search, which always keeps the best $k$ subsets of features at each step during the process of searching. This means that the best discriminating pair of features would be selected by beam search as long as one of them is among the top $k$ discriminating features. When the beam width $k$ is 1, beam search reduces to forward selection. Although beam search has been widely used in neural machine translation systems \citep{sutskever2014sequence}, its effectiveness as a general feature selection algorithm for classification models has rarely been studied. \citet{gupta2002beam} used (modified) beam search to perform feature selection for automatic defect classification, but their result was only based on one specific classifier (support vector machine) on one specific dataset. We provide a more comprehensive evaluation of the effectiveness of beam search for feature selection by considering several different classifiers and datasets.

The rest of this paper is organized as follows. In Section \ref{sec:theory_beam}, we develop and prove some consistency results about the performance of classification models using a subset of features. In Section \ref{sec:algorithm_beam}, we propose to use beam search to perform feature selection, which can be viewed as a generalization of forward selection. In Section \ref{sec:simulation_beam} and \ref{sec:applications_beam}, we apply beam search to simulated and real-world data, by evaluating and comparing the performance of different classification models using different sets of features. In Section \ref{sec:discussion_beam}, we conclude with a discussion.

\section{Problem Formulation and Consistency Result}\label{sec:theory_beam}
Consider the following binary classification problem. Let $X = (x_1, \ldots, x_p) \in \mathbb{R}^p$ and $Y \in \{0, 1\}$. We are given training data $\cD_n = \{(X_1, Y_1), \ldots, (X_n, Y_n)\}$ of i.i.d.\ samples with the same distribution as $(X, Y)$. Suppose we would like to perform classification with feature selection by selecting a subset of features of size $d$ from all $p$ features. Let $\cI$ denote the set of all subsets $I \subset \{1, \ldots, p\}$ of size $d$, and let $X(I) = (x_i)_{i \in I} \in \mathbb{R}^d$. Let $\cC$ denote the family of classifiers $g: \mathbb{R}^d \to \{0, 1\}$.

For a classifier $g \in \cC$ and a subset $I \in \cI$, define the expected risk as
\begin{equation*}
    L(g, I) = \Px(g(X(I)) \neq Y) = \Ex[\mathbf{1}_{g(X(I)) \neq Y}],
\end{equation*}
and the empirical risk as
\begin{equation*}
    \hat{L}_n(g, I) = \frac{1}{n} \sum_{j=1}^n \mathbf{1}_{g(X_j(I))\neq Y_j}.
\end{equation*}
The empirically optimal classifier $g_n^*$ and subset $I_n^*$ are obtained by minimizing the empirical risk:
\begin{equation*}
    g_n^*, I_n^* = \argmin_{g \in \cC, I \in \cI} \hat{L}_n(g, I).
\end{equation*}

For the empirically optimal classifier $g_n^*$ and subset $I_n^*$, the following VC-type inequality holds:

\begin{theorem}\label{thm:vc}
Let $\cS(\cC, n)$ denote the $n$-th shatter coefficient of $\cC$. Then,
\begin{equation*}
    \Px \left( L(g_n^*, I_n^*) - \inf_{g \in \cC, I \in \cI} L(g, I) > \epsilon \right) \le 8\binom{p}{d}\cS(\cC, n)e^{-n\epsilon^2/128}.
\end{equation*}
\end{theorem}

\begin{proof}
For any specific $I \in \cI$, by Theorem 12.6 in \citep{devroye2013probabilistic}, the following classic VC inequality holds:
\begin{equation*}
    \Px \left( \sup_{g \in \cC} |\hat{L}_n(g, I) - L(g, I)| > \epsilon \right) \le 8\cS(\cC, n)e^{-n\epsilon^2/32}.
\end{equation*}
Recall that $\cI$ is the set of all subsets of $\{1 ,\ldots, p\}$ of size $d$, so $|\cI| = \binom{p}{d}$. Therefore we have the following union bound:
\begin{align*}
    \Px \left( \sup_{g \in \cC, I \in \cI} |\hat{L}_n(g, I) - L(g, I)| > \epsilon \right) &\le \sum_{I \in \cI} \Px \left( \sup_{g \in \cC} |\hat{L}_n(g, I) - L(g, I)| > \epsilon \right) \\
    &\le 8\binom{p}{d}\cS(\cC, n) e^{-n\epsilon^2/32}.
\end{align*}
Notice that
\begin{align*}
    L(g_n^*, I_n^*) - \inf_{g \in \cC, I \in \cI} L(g, I) &= L(g_n^*, I_n^*) - \hat{L}_n(g_n^*, I_n^*) + \hat{L}_n(g_n^*, I_n^*) - \inf_{g \in \cC, I \in \cI} L(g, I) \\
    &\le L(g_n^*, I_n^*) - \hat{L}_n(g_n^*, I_n^*) + \sup_{g \in \cC, I \in \cI} |\hat{L}_n(g, I) - L(g, I)| \\
    &\le 2 \sup_{g \in \cC, I \in \cI} |\hat{L}_n(g, I) - L(g, I)|.
\end{align*}
Hence
\begin{align*}
    \Px \left( L(g_n^*, I_n^*) - \inf_{g \in \cC, I \in \cI} L(g, I) > \epsilon \right) &\le \sum_{I \in \cI} \Px \left( \sup_{g \in \cC, I \in \cI} |\hat{L}_n(g, I) - L(g, I)| > \epsilon/2 \right) \\
    &\le 8\binom{p}{d}\cS(\cC, n)e^{-n\epsilon^2/128}.
    \qedhere
\end{align*}
\end{proof}

This gives the following corollary about consistency:

\begin{corollary}\label{cor:consistency}
\begin{equation*}
    \Ex[L(g_n^*, I_n^*)] - \inf_{g \in \cC, I \in \cI} L(g, I) \le 16 \sqrt{\frac{\log(8\binom{p}{d}eS(\cC, n))}{2n}}.
\end{equation*}
\end{corollary}

\begin{proof}
For simplicity of notation, let $Z = L(g_n^*, I_n^*) - \inf_{g \in \cC, I \in \cI} L(g, I) \ge 0$. Theorem \ref{thm:vc} is equivalent to
\begin{equation*}
    P(Z > \epsilon) \le 8\binom{p}{d}\cS(\cC, n)e^{-n\epsilon^2/128}.
\end{equation*}
Notice that for any $u > 0$, we have
\begin{align*}
    \Ex[Z^2] &= \int_0^\infty \Px(Z^2 > t) dt \\
    &= \int_0^u \Px(Z^2 > t) dt + \int_u^\infty \Px(Z^2 > t) dt \\
    &\le u + \int_u^\infty 8\binom{p}{d}\cS(\cC, n)e^{-nt/128} dt \\
    &= u + 1024\binom{p}{d}\cS(\cC, n) \frac{e^{-nu/128}}{n}.
\end{align*}
Since the above inequality holds for any $u > 0$, we choose the $u$ that minimizes the right hand side. Plugging in $u = 128\log(8\binom{p}{d}\cS(C,n))/n$, and the inequality becomes
\begin{equation*}
    \Ex[Z^2] \le \frac{128\log(8\binom{p}{d}\cS(C,n))}{n} + \frac{128}{n} = \frac{128\log(8\binom{p}{d}e\cS(C,n))}{n}.
\end{equation*}
Consequently,
\begin{equation*}
    \Ex[L(g_n^*, I_n^*)] - \inf_{g \in \cC, I \in \cI} L(g, I) = \Ex[Z] \le \sqrt{\Ex[Z^2]} \le 16 \sqrt{\frac{\log(8\binom{p}{d}eS(\cC, n))}{2n}}.
    \qedhere
\end{equation*}
\end{proof}

We can characterize the optimal classifier $g^*$ and subset $I^*$ that minimize the expected risk $L(g, I)$ if we make an additional assumption.

\begin{theorem}
Assume that there exist a subset $I^*$ of size $d$ such that
\begin{equation*}
    \Ex[Y \mid X] = \Ex[Y \mid X(I^*)],
\end{equation*}
For any $x \in \mathbb{R}^d$, let
\begin{equation*}
    \eta(x) = \Ex[Y \mid X(I^*) = x],
\end{equation*}
and 
\begin{equation*}
    g^*(x) = 
    \begin{cases}
    1 & \text{if } \eta(x) > 1/2, \\
    0 & \text{otherwise}.
    \end{cases}
\end{equation*}
Then for any classifier $g$ and subset $I$ of size $d$, we have
\begin{equation*}
    L(g^*, I^*) \le L(g, I).
\end{equation*}
Equivalently,
\begin{equation*}
    L(g^*, I^*) = \inf_{g \in \cC, I \in \cI} L(g, I).
\end{equation*}
\end{theorem}

\begin{proof}
Recall that $L(g, I) = \Px(g(X(I)) \neq Y)$. For any classifier $g$ and subset $I$, let $X(I \cup I^*) = x^+ \in \mathbb{R}^{|I \cup I^*|}$, $X(I) = x' \in \mathbb{R}^d$, and $X(I^*) = x \in \mathbb{R}^d$. The conditional error probability of classifier $g$ on subset $I$ given $X(I \cup I^*) = x^+$ can be expressed as
\begin{align*}
    &\Px(g(X(I)) \neq Y \mid X(I \cup I^*) = x^+) \\
    &= 1 - \Px(g(X(I)) = Y \mid X(I \cup I^*) = x^+) \\
    &= 1 - \Px(g(X(I)) = 1, Y = 1 \mid X(I \cup I^*) = x^+) - \Px(g(X(I)) = 0, Y = 0 \mid X(I \cup I^*) = x^+) \\
    &= 1 - \Px(g(X(I)) = 1 \mid X(I \cup I^*) = x^+) \cdot \Px(Y = 1 \mid X(I \cup I^*) = x^+) \\
    &- \Px(g(X(I)) = 0 \mid X(I \cup I^*) = x^+) \cdot \Px(Y = 0 \mid X(I \cup I^*) = x^+)] \\
    &= 1 - \Px(g(X(I)) = 1 \mid X(I) = x') \cdot \Px(Y = 1 \mid X(I^*) = x) \\
    &- \Px(g(X(I)) = 0 \mid X(I) = x') \cdot \Px(Y = 0 \mid X(I^*) = x) \\
    &= 1 - \mathbf{1}_{g(x') = 1} \eta(x) - \mathbf{1}_{g(x') = 0} (1 - \eta(x)).
\end{align*}
Similarly, 
\begin{align*}
    &\Px(g^*(X(I^*)) \neq Y \mid X(I \cup I^*) = x^+) \\
    &= 1 - \Px(g^*(X(I^*)) = Y \mid X(I \cup I^*) = x^+) \\
    &= 1 - \Px(g^*(X(I^*)) = 1, Y = 1 \mid X(I \cup I^*) = x^+) - \Px(g^*(X(I^*)) = 0, Y = 0 \mid X(I \cup I^*) = x^+) \\
    &= 1 - \Px(g^*(X(I^*)) = 1 \mid X(I \cup I^*) = x^+) \cdot \Px(Y = 1 \mid X(I \cup I^*) = x^+) \\
    &- \Px(g^*(X(I^*)) = 0 \mid X(I \cup I^*) = x^+) \cdot \Px(Y = 0 \mid X(I \cup I^*) = x^+)] \\
    &= 1 - \Px(g^*(X(I^*)) = 1 \mid X(I^*) = x) \cdot \Px(Y = 1 \mid X(I^*) = x) \\
    &- \Px(g^*(X(I^*)) = 0 \mid X(I^*) = x) \cdot \Px(Y = 0 \mid X(I^*) = x) \\
    &= 1 - \mathbf{1}_{g^*(x) = 1} \eta(x) - \mathbf{1}_{g^*(x) = 0} (1 - \eta(x)).
\end{align*}
Therefore
\begin{align*}
    &\Px(g(X(I)) \neq Y \mid X(I \cup I^*) = x^+) - \Px(g^*(X(I^*)) \neq Y \mid X(I \cup I^*) = x^+) \\
    &= \eta(x) (\mathbf{1}_{g^*(x) = 1} - \mathbf{1}_{g(x') = 1}) + (1 - \eta(x)) (\mathbf{1}_{g^*(x) = 0} - \mathbf{1}_{g(x') = 0}) \\
    &= (2\eta(x) - 1) (\mathbf{1}_{g^*(x) = 1} - \mathbf{1}_{g(x') = 1}) \\
    &\ge 0
\end{align*}
by the definition of $g^*(x)$. Integrating both sides with respect to the distribution of $x^+$, we have
\begin{equation*}
    L(g, I) = \Px(g(X(I)) \neq Y) \ge \Px(g^*(X(I^*)) \neq Y) = L(g^*, I^*)
\end{equation*}
for any classifier $g$ and subset $I$.
\end{proof}

\section{Algorithm}\label{sec:algorithm_beam}
One simple approach to select the subset of features $I$ of size $d$ from all $p$ features is forward selection: adding features into the set $I$ one by one, and at each step $t$ selecting the feature $i$ that leads to the lowest misclassification rate using only the previously selected $t-1$ features and feature $i$. A more general approach is beam search, which works as follows:

\begin{algorithm}[ht]
\begin{enumerate}
    \item Set the beam width $k$.
    \item Increase the step $t$ from 1 to $d$. At each step $t$:
    \begin{enumerate}
        \item If $t=1$, then fit $p$ models using each of the $p$ features individually, and select the best $k$ features corresponding to the best $k$ models in terms of misclassification rate. These $k$ features are viewed as $k$ groups of features of size 1.
        \item Otherwise, for each of the $k$ groups of features of size $t-1$, fit $p-(t-1)$ models by adding each of the remaining $p-(t-1)$ features into the group of features. In total $k(p-(t-1))$ models are fitted, each corresponding to $k(p-(t-1))$ groups of features of size $t$. Among them, select the best $k$ groups of features corresponding to the best $k$ models in terms of misclassification rate.
    \end{enumerate}
    \item Finally, out of the $k$ groups of features of size $d$, select the best group of features corresponding to the model with the lowest misclassification rate.
\end{enumerate}
\caption{Beam search for feature selection}
\label{algo:classification}
\end{algorithm}

The above beam search algorithm for feature selection can be viewed as a generalization of forward selection, and the main difference is that beam search always keeps $k$ candidates during the process of searching, whereas forward selection only keeps 1 candidate. This means that beam search with beam width $k > 1$ is not as greedy as forward selection. For example, the best discriminating pair of features would be selected by beam search as long as one of them is among the top $k$ discriminating features. When the beam width $k = 1$, beam search reduces to forward selection.

\section{Simulation Studies}\label{sec:simulation_beam}
In this section, we evaluate and compare the performance of different classification models using different sets of features on simulated data. More specifically, we consider the following five widely used classification models: $k$-nearest neighbors (KNN) \citep{cover1967nearest}, linear discriminant analysis (LDA) \citep{hastie2009elements}, quadratic discriminant analysis (QDA) \citep{hastie2009elements}, support vector machine (SVM) \citep{cortes1995support}, and logistic regression with $L_1$ regularization \citep{hastie2009elements}. In terms of features, we consider the following three cases: all features, features selected by forward selection, and features selected by beam search. The evaluation metric is the misclassification rate. The number of neighbors in KNN is set to be 15. For logistic regression with $L_1$ regularization, $\lambda$ is selected using cross-validation on the training set. 

We generate simulated data in three different settings. For each setting, we perform 50 simulations, and report the means and standard errors of the misclassification rates. In each simulation, we generate two $n \times p$ data matrices, one as training set and the other as test set. For both of them, the first $n/2$ rows are labeled as positive class, and the last $n/2$ rows are labeled as negative class.

\subsection{Simulation 1}
In the first simulation, the two data matrices $\mathbf{X}_{train}$ and $\mathbf{X}_{test}$ are generated in the following way:
\begin{enumerate}
\item Two $m$-dimensional vectors $u$ and $v$ are generated, and their entries satisfy $u_j \sim U(0,1)$ and $v_j \sim U(-1,0)$ for $1 \le j \le m$.
\item The entries $\mathbf{X}_{ij}$ in $\mathbf{X}_{train}$ and $\mathbf{X}_{test}$ are independent, and they satisfy
\begin{align*}
\mathbf{X}_{ij} =
\begin{cases}
N(u_j,1), &\text{for } 1 \le j \le m \text{ and } 1 \le i \le n/2, \\
N(v_j,1), &\text{for } 1 \le j \le m \text{ and } n/2 < i \le n, \\
N(0,1), &\text{for } j > m.
\end{cases}
\end{align*} 
\end{enumerate}
In other words, the first $m$ columns contain useful signals, where the first $n/2$ rows and the last $n/2$ rows differ in their means. The rest of the $p-m$ columns are just i.i.d.\ standard Gaussian noise. We set $n = 500$, $p = 100$, and $m = 5$. For feature selection algorithms, we set the number of selected features $d=5$, and the beam width $k=5$.

\begin{table}[ht]
\centering
\begin{tabular}{ c c c c c c }
    \hline
    & KNN & LDA & QDA & SVM & Logistic \\
    \hline
    all features & 0.210(0.009) & 0.152(0.007) & 0.355(0.006) & 0.172(0.007) & 0.142(0.006) \\
    forward selection & 0.157(0.008) & 0.146(0.008) & 0.144(0.007) & 0.165(0.009) & 0.138(0.007) \\
    beam search & 0.151(0.008) & 0.137(0.006) & 0.136(0.008) & 0.160(0.009) & 0.134(0.007) \\
    \hline
\end{tabular}
\caption{The means (and standard errors) of the misclassification rate for Simulation 1 over 50 simulations.}
\label{tab:sim1_beam}
\end{table}

The misclassification rates are reported in Table \ref{tab:sim1_beam}. We observe that the classification models trained with only $d=5$ features selected by either forward selection or beam search all perform better than the classification models trained with all $p=100$ features, although the amount of improvement would depend on the specific model. This result agrees with our expectation that when there are many irrelevant features, performing feature selection using either forward selection or beam search could identify the relevant features and lead to improved predictive performance. However, the misclassification rates in the second row are extremely close to those in the third row, indicating that forward selection and beam search perform similarly under this setting. This is also not surprising, because all $m=5$ discriminative features are independent, so even an extremely greedy feature selection algorithm like forward selection could correctly pick out the discriminative features.

\subsection{Simulation 2}
In the second simulation, the two data matrices $\mathbf{X}_{train}$ and $\mathbf{X}_{test}$ are generated in the following way:
\begin{enumerate}
\item The first two columns $(\mathbf{X}_{i1},  \mathbf{X}_{i2})$ in $\mathbf{X}_{train}$ and $\mathbf{X}_{test}$ are jointly normal, and they satisfy
\begin{align*}
(\mathbf{X}_{i1},  \mathbf{X}_{i2}) =
\begin{cases}
N
\left(
(0,0), 
\left(
\begin{smallmatrix}
1 & 0.9 \\
0.9 & 1 
\end{smallmatrix}
\right)
\right),
&\text{for } 1 \le i \le n/2, \\
N
\left(
(0,0), 
\left(
\begin{smallmatrix}
1 & -0.9 \\
-0.9 & 1 
\end{smallmatrix}
\right)
\right),
&\text{for } n/2 < i \le n.
\end{cases}
\end{align*} 
\item The entries $\mathbf{X}_{ij}$ in the rest of the columns in $\mathbf{X}_{train}$ and $\mathbf{X}_{test}$ are independent, and they satisfy
\begin{align*}
\mathbf{X}_{ij} =
\begin{cases}
N(0.3,1), &\text{for } 3 \le j \le 4 \text{ and } 1 \le i \le n/2, \\
N(-0.3,1), &\text{for } 3 \le j \le 4 \text{ and } n/2 < i \le n, \\
N(0,1), &\text{for } j > 4.
\end{cases}
\end{align*} 
\end{enumerate}
In other words, the first four columns contain useful signals, and the rest of the $p-4$ columns are just i.i.d.\ standard Gaussian noise. Among the first four columns, the third and fourth columns are the most discriminative features when considered individually, where the first $n/2$ rows and the last $n/2$ rows differ in their means. The first and second columns have no discriminative power when considered individually, but when considered jointly, they are the most discriminating pair of features, where the first $n/2$ rows and the last $n/2$ rows differ in their correlations. Note that the decision boundary for the first and second columns is nonlinear, so only classification models with nonlinear decision boundary (KNN, QDA, SVM) would be able to detect this pattern. We set $n = 500$ and $p = 10$. For feature selection algorithms, we set the number of selected features $d=2$, and the beam width $k=5$.

\begin{table}[ht]
\centering
\begin{tabular}{ c c c c c c }
    \hline
    & KNN & LDA & QDA & SVM & Logistic \\
    \hline
    all features & 0.226(0.003) & 0.342(0.002) & 0.126(0.002) & 0.202(0.003) & 0.340(0.003) \\
    forward selection & 0.380(0.005) & 0.353(0.004) & 0.354(0.006) & 0.358(0.004) & 0.353(0.005) \\
    beam search & 0.220(0.015) & 0.352(0.004) & 0.184(0.013) & 0.218(0.013) & 0.359(0.005) \\
    \hline
\end{tabular}
\caption{The means (and standard errors) of the misclassification rate for Simulation 2 over 50 simulations.}
\label{tab:sim2_beam}
\end{table}

The misclassification rates are reported in Table \ref{tab:sim2_beam}. We observe that for classification models that have nonlinear decision boundary (KNN, QDA, SVM), they perform much better using features selected by beam search compared to using features selected by forward selection. This result demonstrates the advantage of beam search over forward selection as a feature selection algorithm: the ability to consider two or more features jointly instead of considering features individually. More specifically, under this setting forward selection will pick the third and fourth columns, because they are the most discriminative features when considered individually, and any pair of features that doesn't include the third or fourth column will not be considered. However, for classification models that have nonlinear decision boundary (KNN, QDA, SVM), beam search will pick the first and second columns, because when considered jointly they are the most discriminating pair of features, and this pair of features will be a candidate considered by beam search with suitable beam width.

\subsection{Simulation 3}
In the third simulation, the two data matrices $\mathbf{X}_{train}$ and $\mathbf{X}_{test}$ are generated in the following way:
\begin{enumerate}
\item The first two columns $(\mathbf{X}_{i1},  \mathbf{X}_{i2})$ in $\mathbf{X}_{train}$ and $\mathbf{X}_{test}$ are generated jointly, and they satisfy
\begin{align*}
(\mathbf{X}_{i1},  \mathbf{X}_{i2}) =
\begin{cases}
\text{Uniform in the set } \{(x, y): -3<x<3, -3<y<3, x+y > -0.2 \},
&\text{for } 1 \le i \le n/2, \\
\text{Uniform in the set } \{(x, y): -3<x<3, -3<y<3, x+y < 0.2 \},
&\text{for } n/2 < i \le n.
\end{cases}
\end{align*} 
\item The entries $\mathbf{X}_{ij}$ in the rest of the columns in $\mathbf{X}_{train}$ and $\mathbf{X}_{test}$ are independent, and they satisfy
\begin{align*}
\mathbf{X}_{ij} =
\begin{cases}
U(-1,3), &\text{for } 3 \le j \le 4 \text{ and } 1 \le i \le n/2, \\
U(-3,1), &\text{for } 3 \le j \le 4 \text{ and } n/2 < i \le n, \\
N(0,1), &\text{for } j > 4.
\end{cases}
\end{align*} 
\end{enumerate}
In other words, the first four columns contain useful signals, and the rest of the $p-4$ columns are just i.i.d.\ standard Gaussian noise. Among the first four columns, the first and second columns are less discriminative than the third and fourth columns when considered individually, but when considered jointly, they are the most discriminating pair of features. Note that in this simulation, the decision boundary for the first and second columns is linear, so we expect all classification models to perform better using features selected by beam search compared to using features selected by forward selection. We set $n = 500$ and $p = 10$. For feature selection algorithms, we set the number of selected features $d=2$, and the beam width $k=5$.

\begin{table}[ht]
\centering
\begin{tabular}{ c c c c c c }
    \hline
    & KNN & LDA & QDA & SVM & Logistic \\
    \hline
    all features & 0.038(0.001) & 0.031(0.001) & 0.033(0.001) & 0.036(0.001) & 0.021(0.001) \\
    forward selection & 0.117(0.005) & 0.111(0.006) & 0.116(0.005) & 0.116(0.005) & 0.095(0.006) \\
    beam search & 0.077(0.003) & 0.071(0.004) & 0.071(0.004) & 0.070(0.004) & 0.065(0.002) \\
    \hline
\end{tabular}
\caption{The means (and standard errors) of the misclassification rate for Simulation 3 over 50 simulations.}
\label{tab:sim3_beam}
\end{table}

The misclassification rates are reported in Table \ref{tab:sim3_beam}. As expected, we see that all classification models obtain lower misclassification rates using features selected by beam search compared to using features selected by forward selection. This result again demonstrates that beam search could outperform forward selection when the features are correlated so that they have more discriminative power when considered jointly than individually.

\section{Applications}\label{sec:applications_beam}
In this section, we apply beam search to three gene expression datasets, all of which were proposed and preprocessed by \citet{de2008clustering}. In all datasets, the rows represent different samples, and the columns represent different genes. We evaluate and compare the performance of four different classification models (KNN, LDA, SVM, logistic regression with $L_1$ regularization) using three different sets of features (all features, features selected by forward selection, features selected by beam search). QDA is not included in the comparison because the sample sizes of all three datasets are too small to run QDA. For all datasets, we set $d=10$ as the number of selected features, and $k=5$ as the beam width, as this combination achieves a good compromise between model performance and run time. We perform 5-fold cross validation on the datasets, and report the means and standard errors of the misclassification rates.

\subsection{Breast and Colon Cancer Gene Expression Dataset}
The first dataset consists of $n=104$ samples and $p=182$ genes, and the samples have already been classified into two groups: 62 samples correspond to breast cancer tissues, and 42 samples correspond to colon cancer tissues.

\begin{table}[ht]
\centering
\begin{tabular}{ c c c c c }
    \hline
    & KNN & LDA & SVM & Logistic \\
    \hline
    all features & 0.069(0.026) & 0.107(0.034) & 0.030(0.012) & 0.030(0.020) \\
    forward selection & 0.077(0.019) & 0.050(0.039) & 0.058(0.011) & 0.078(0.025) \\
    beam search & 0.068(0.012) & 0.039(0.010) & 0.040(0.019) & 0.050(0.039) \\
    \hline
\end{tabular}
\caption{The means (and standard errors) of the misclassification rate on the breast and colon cancer gene expression dataset.}
\label{tab:bc_beam}
\end{table}

The misclassification rates are reported in Table \ref{tab:bc_beam}. For all four classification models, we observe that using features selected by beam search achieves lower misclassification rates compared to using features selected by forward selection. In addition, we could also notice that classification models could obtain similar or better performance using only $d=10$ features selected by beam search compared to using all $p=182$ features. In particular, LDA performs much better using only $d=10$ features selected by beam search compared to using all $p=182$ features.

\subsection{Prostate Cancer Gene Expression Dataset}
The second dataset consists of $n=102$ samples and $p=339$ genes, and the samples have already been classified into two groups: 52 samples correspond to prostate cancer tissues, and 50 samples correspond to normal prostate tissues.

\begin{table}[ht]
\centering
\begin{tabular}{ c c c c c }
    \hline
    & KNN & LDA & SVM & Logistic \\
    \hline
    all features & 0.205(0.034) & 0.077(0.028) & 0.086(0.026) & 0.068(0.047) \\
    forward selection & 0.174(0.046) & 0.225(0.035) & 0.195(0.034) & 0.136(0.025) \\
    beam search & 0.177(0.035) & 0.105(0.043) & 0.157(0.040) & 0.128(0.041) \\
    \hline
\end{tabular}
\caption{The means (and standard errors) of the misclassification rate on the prostate cancer gene expression dataset.}
\label{tab:prostate_beam}
\end{table}

The misclassification rates are reported in Table \ref{tab:prostate_beam}. Again, we observe that classification models perform similarly (KNN and logistic regression with $L_1$ regularization) or much better (LDA and SVM) when using features selected by beam search compared to using features selected by forward selection. In addition, KNN achieves lower misclassification rates using only $d=10$ features selected by beam search, although other classification models (LDA, SVM, and logistic regression with $L_1$ regularization) all achieve their lowest misclassification rates using all $p=339$ features. This is not surprising, since $d=10$ features is less than three percent of $p=339$ features, and therefore it might not be enough to contain all the useful signals. Still, the interpretability of the selected features is a key advantage of using beam search to perform feature selection. In this example, the biological relationship between the selected genes and prostate cancer might be worthy of further investigation.

\subsection{Lung Cancer Gene Expression Dataset}
The third dataset consists of $n=203$ samples and $p=1543$ genes, and the samples have already been classified into five groups: 186 samples correspond to four different kinds of lung cancer tissues, and 17 samples correspond to normal lung tissues.

\begin{table}[ht]
\centering
\begin{tabular}{ c c c c c }
    \hline
    & KNN & LDA & SVM & Logistic \\
    \hline
    all features & 0.092(0.010) & 0.056(0.015) & 0.092(0.021) & 0.096(0.012) \\
    forward selection & 0.137(0.026) & 0.121(0.031) & 0.178(0.014) & 0.137(0.027) \\
    beam search & 0.131(0.019) & 0.117(0.017) & 0.127(0.014) & 0.107(0.031) \\
    \hline
\end{tabular}
\caption{The means (and standard errors) of the misclassification rate on the lung cancer gene expression dataset.}
\label{tab:lung_beam}
\end{table}

The misclassification rates are reported in Table \ref{tab:lung_beam}. Once again, we see that classification models perform similarly (KNN and LDA) or much better (SVM and logistic regression with $L_1$ regularization) when using features selected by beam search compared to using features selected by forward selection. Also, all four classification models achieve their lowest misclassification rates using all $p=1543$ features, which is likely due to the fact that $d=10$ features, being less than one percent of $p=1543$ features, are not enough to retain all the discriminative power. Nevertheless, this example demonstrates that classification models could obtain slightly worse performance using an extremely small proportion of features selected by beam search. Combined with the interpretability of the selected features, these properties make beam search a useful feature selection algorithm when applied to real-world datasets.

\section{Discussion}\label{sec:discussion_beam}
In this paper, we have presented and proved some consistency results about the performance of classification models using a subset of features. In addition, we have proposed to use beam search to perform feature selection, which can be viewed as a generalization of forward selection. We have applied beam search to both simulated and real-world data, by evaluating and comparing the performance of different classification models using different sets of features. The results have demonstrated that beam search could outperform forward selection, especially when the features are correlated so that they have more discriminative power when considered jointly than individually. Moreover, in some cases classification models could obtain comparable performance using only ten features selected by beam search instead of hundreds of original features.

In the future, we plan to extend the exploration of beam search as a general feature selection algorithm to regression problems. In this paper we focused on classification problems, but it is also possible to modify our beam search algorithm to run on regression problems: instead of misclassification rate, another metric suitable for regression problems should be used to rank and select different subsets of features.

Another interesting direction worth pursuing is to consider integrating a good stopping rule into the beam search algorithm. Currently the number of selected features $d$ is fixed and serves as a hyperparameter. A better approach is to implement a stopping rule so that the beam search algorithm would stop adding more features when the marginal improvement in model performance is small. Doing so would eliminate the need to specify a fixed number of selected features in advance, and the number of features that are selected would depend on the specific dataset and model.

\bibliography{ref}

\end{document}